\documentclass[letterpaper, 10 pt, conference]{ieeeconf}  

\IEEEoverridecommandlockouts                              

\usepackage{cite}
\usepackage{amsmath,amssymb}
\usepackage{algorithmic}
\usepackage{graphicx}
\usepackage{textcomp}
\usepackage{xcolor}
\graphicspath{ {./figures/} }

\DeclareMathOperator*{\argmin}{arg\,min}

\newcommand\independent{\protect\mathpalette{\protect\independenT}{\perp}}
\def\independenT#1#2{\mathrel{\rlap{$#1#2$}\mkern2mu{#1#2}}}

\title{\LARGE \bf
Unsupervised Feature Learning for Manipulation with Contrastive Domain Randomization
}

\author{Carmel Rabinovitz$^{1}$ Niko Grupen$^{2}$ and Aviv Tamar$^{1}$
\thanks{$^{1}$Department of Electrical Engineering,
        Technion, Israel       
        {\tt\small carmelrab@campus.technion.ac.il,
        avivt@technion.ac.il}}%
\thanks{$^{2}$Department of Computer Science, Cornell University,
        NY, USA
        {\tt\small niko@cs.cornell.edu}}%
}

\begin{document}

\maketitle
\thispagestyle{empty}
\pagestyle{empty}

\begin{abstract}
Robotic tasks such as manipulation with visual inputs require image features that capture the physical properties of the scene, e.g., the position and configuration of objects. Recently, it has been suggested to learn such features in an unsupervised manner from simulated, self-supervised, robot interaction; the idea being that high-level physical properties are well captured by modern physical simulators, and their representation from visual inputs may transfer well to the real world. In particular, learning methods based on noise contrastive estimation have shown promising results.
To robustify the simulation-to-real transfer, domain randomization (DR) was suggested for learning features that are invariant to irrelevant visual properties such as textures or lighting. In this work, however, we show that a naive application of DR to unsupervised learning based on contrastive estimation \emph{does not} promote invariance, as the loss function maximizes mutual information between the features and both the relevant and irrelevant visual properties. We propose a simple modification of the contrastive loss to fix this, exploiting the fact that we can control the simulated randomization of visual properties. Our approach learns physical features that are significantly more robust to visual domain variation, as we demonstrate using both rigid and non-rigid objects.
\end{abstract}

\IEEEpeerreviewmaketitle

\section{Introduction}

If a robot is to perform object manipulation tasks, it must process its sensory inputs to extract physical properties that are relevant to the objects and the task. For example, for picking an object, the robot may extract the pose of the object and its size from a camera image. However, for general objects and tasks, defining the relevant properties can be difficult -- what properties should be extracted for manipulating a deformable object such as a rope? And what should be encoded about objects with nontrivial geometries? 

If the task (or task distribution) is known in advance and can be described using a reward function, reinforcement learning (RL) can be used to learn a policy for solving the task~\cite{levine2016end}. By definition, however, the RL agent will only learn features that are relevant for the task it trained on, as specified through the reward function, and the features may not be applicable for other tasks. In addition, for many real world tasks it is difficult to define a reward function~\cite{abbeel2004apprenticeship, christiano2017deep, singh2019end}.

An alternative and promising approach is to learn features in an unsupervised fashion, from self-supervised data of the robot interacting with objects. The idea is that features that are useful for making predictions about the future state of the environment, likely encode relevant object properties, and will therefore be useful for a variety of downstream tasks. A popular method for learning such features is based on contrastive learning: maximizing mutual information between features of past and future observations~\cite{oord2018representation, tian2019contrastive, chen2020simple}.

However, collecting a diverse set of self-supervised robot interactions requires significant resources~\cite{levine2018learning,pinto2016supersizing,dasari2019robonet}. Exploiting the recent advances in physical simulators, Yan et al.~proposed to learn physical features using contrastive learning from \textit{simulated} interaction data \cite{yan2020learning}. Since simulated data is easy to generate, and no reward function needs to be manually defined for unsupervised learning, the approach in \cite{yan2020learning} offers a principled and scalable method for learning general physical features, such as the configuration of non-rigid objects. 

The features learned in simulation can only be accurate up to the differences between simulation and reality -- the \textit{sim-to-real gap}. To robustify learning, a popular trick is domain randomization (DR) -- randomly modify irrelevant visual properties such as lighting and object textures during training, with the hope of learning features that are invariant to these perturbations. Indeed, \cite{yan2020learning} demonstrated successful sim-to-real transfer using DR. Conversely, however, in this work we claim that a straightforward application of DR to contrastive learning is fundamentally flawed: by carefully formulating the problem, we show that the standard CPC loss function aims to maximize information between the features and both relevant and irrelevant visual properties, and therefore does not learn features that are robustly invariant to the irrelevant randomization. Indeed, the experiments in \cite{yan2020learning} were designed with a very small sim-to-real gap, and we demonstrate that when this gap is increased, the quality of the features degrade significantly. Our derivation also prescribes a simple fix -- by applying a different randomization to past and future observations, which can easily be done in simulation, we guarantee that invariant representations are learned.  

We demonstrate that our method, termed Contrastive Domain Randomization (CDR), is able to learn relevant physical features of rigid and non-rigid objects that are highly robust to irrelevant visual perturbations such as background, texture, and lighting. As such, CDR paves the way for learning visual representations of physical properties that are general, robust, and easy to obtain.

\section{Related Work}
Our work is situated between three areas: domain randomization, contrastive learning, and object manipulation.

\subsection{Domain Randomization}
Domain randomization is a class of domain adaptation techniques~\cite{wang2018deep, patel2015visual} that targets the transition of learned models from simulation to the real-world. DR has been successful on vision problems (e.g. object detection, pose estimation \cite{tobin2017domain}) often using a rendering engine to perturb non-essential properties in synthetic scenes~\cite{ren2019domain}. Recent work has extended DR to robotic grasping \cite{tobin2018grasping}, navigation~\cite{sadeghi2016cad2rl}, locomotion~\cite{mordatch2015ensemble}, and, relevantly, deformable object manipulation \cite{matas2018sim}. Further methods alter simulation dynamics in addition to visual features~\cite{antonova2017reinforcement, peng2018sim}, randomizing physical properties, such as friction and latency, to learn robust manipulation policies.
Implicit in DR is an assumption that the environment is composed of relevant and irrelevant properties, which together generate an observation. The goal is to learn a representation that attends to relevant properties (e.g. physics) while ignoring the irrelevant ones. In this work, we make a connection between the domain randomization objective and the idea of learning invariances, as studied in the causal inference literature~\cite{peters2017elements}. Recent work on invariant risk minimization~\cite{arjovsky2019invariant} also aims to learn invariant features. However, our setting controls the simulated domain (the `intervention' in causal inference parlance), enabling more effective representation learning.

\subsection{Contrastive Learning}
Contrastive learning has gained popularity as a method for self-supervised feature learning, and relates to information-theoretic concepts such as mutual information~\cite{oord2018representation, tian2019contrastive, chen2020simple}. Leveraging spatially- or temporally-coherent data (e.g., images or videos), contrastive learning techniques construct "pretext" tasks~\cite{misra2020self} in which a neural network must discriminate between similar and dissimilar input samples~\cite{chen2020big, chen2020improved}. Pretext tasks can involve predicting image patches~\cite{henaff2019data, oord2018representation}, frames in videos~\cite{oord2018representation}, optical flow~\cite{tian2019contrastive}, programs \cite{jain2020contrastive}, and other visual features~\cite{he2020momentum, chen2020improved, chen2020big}. Crucially, contrastive learning can be interpreted as learning a non-linear transformation that is invariant to distortions of the input data \cite{hadsell2006dimensionality}. Our method combines these ideas with domain-randomization to learn self-supervised domain-invariant features.

\subsection{Deformable Object Manipulation}
We focus on deformable object manipulation approaches that combine predictive models with forward planning. 
Recent work has proposed learning non-linear deformation functions to assist feedback control \cite{hu2018three, jia2018learning}. Generative modeling \cite{chen2016infogan, finn2016unsupervised} has enabled neural networks to learn richer dynamics that can be used by model-predictive control (MPC) for manipulation~\cite{ebert2018visual, kurutach2018learning}. Contrastive learning methods, which represent spatial/temporal patterns in the data as latent vectors, are particularly useful for model-based planning \cite{wang2019learning, nair2017combining}. For example, recent methods applied MPC for cloth and rope manipulation, using a forward model learned using contrastive learning~\cite{ding2020mutual, yan2020learning}. Our work extends these approaches to planning that is \textit{invariant} to irrelevant visual properties.

\section{Background}

Our work builds on contrastive learning and domain randomization. To set the stage for our development, we formalize both ideas here under a unified notation.

\subsection{Contrastive Learning} \label{sub: Contrastive Learning}

Contrastive Learning methods learn compact representations of high dimensional observations such as images, video, or audio. Unlike supervised learning, where a parametric model $f$ is trained to maximize the similarity between a prediction $\hat{z}$ and a known label $y$, contrastive learning trains an encoder $f_\theta$, parametrized by $\theta$, to map observations $o$ into a latent representation $\hat{z} = f_\theta(o)$, and makes predictions directly in the latent space. For example, Contrastive Predictive Coding (CPC) \cite{oord2018representation} introduces a self-supervised instance discrimination task in which $f_\theta$ is used to predict a single positive label $y$ amongst $N-1$ contrastive labels $y_j$. This objective, known as InfoNCE loss, simultaneously maximizes the similarity $h(\hat{z}, y)$ between an observation based prediction $\hat{z}$ and the positive label $y$ while minimizing the similarity $h(\hat{z}, y_j)$ between $\hat{z}$ and contrastive labels $y_j$. As a self-supervised task, positive labels $y$ are computed from the observation itself, while contrastive labels $y_j$ are computed from different observations in the data. InfoNCE frames the discrimination task as a standard classification problem, through the cross-entropy loss:

\begin{equation}
    \mathcal{L_{\text{INCE}}^N} =  \mathbb{E}_o \mathbb{E}_{y|o} \left[-\log \frac{h(\hat{z},y)}{\sum_{j=1}^{N} h(\hat{z},y_j )}\right], 
    \label{CPC}
\end{equation}

\noindent where $N$ represents the total number of positive and negative samples. Prior work showed that minimizing \eqref{CPC} is equivalent to maximizing the mutual information (MI) between the observation space $o$ and the latent space $z$ according to the following inequality \cite{oord2018representation, tian2019contrastive}: $ I(o, \hat{z}) \geq \log(N) - \mathcal{L}_{\text{INCE}}^N$. Note that increasing the number of contrastive samples in $N$ improves the lower bound on mutual information, and in turn, the learnt representations.

Typically, the InfoNCE loss is minimized using stochastic gradient descent, where a batch contains random observations from the data, and the loss for a batch of size $N$ is:
\begin{equation*}
    \sum_{k=1}^{N} -\log  \frac{h(\hat{z}_i,y_i)}{\sum_{j = 1}^N h(\hat{z}_i,y_j))}.
\end{equation*}
That is, we set the constrastive labels for an observation to be the labels of other observations in the batch.

Contrastive samples often leverage spatial, temporal, or semantic consistency in the input data. Given image data $o$, for example, it is possible to generate a prediction $\hat{z} = f_\theta(\tau'(o))$ and positive label $y = f_\theta(\tau(o))$ by applying separate visual transformations ($\tau'$ and $\tau$, respectively) to a single image $o$. Contrastive labels can then be generated by applying transformations to different images $y_j = f_\theta(\tau(o_j))$ \cite{chen2020simple, chen2020big}. For sequential data, the temporal dimension of the input serves as a useful signal with which to generate contrastive labels. CPC follows this approach, defining the instance discrimination task over frames in a video~\cite{oord2018representation}. Specifically, from the latent representation $z_t$ of the image at time $t$, $o_t$, CPC predicts the latent representation of a frame $k$ steps in the future $\hat{z}_{t+k}$, using the representation of the true frame $z_{t+k}$ as a positive label, and representations of frames from random time-steps (or different videos altogether) as contrastive labels $z_{j \neq t+k}$. In doing so, CPC simultaneously maximizes the similarity $h(\hat{z}_{t+k}, z_{t+k})$ between the predicted and true future frame representations and minimizes the similarity $h(\hat{z}_{t+k}, z_{j \neq t+k})$ between the prediction and each of the contrastive sample representations. Key to this approach is the joint optimization of an encoder $z_t = f_\theta(o_t)$ and an auto-regressive model $\hat{z}_{t+k} = g_\phi(z_{\leq t})$, parametrized by $\phi$, that at time $t$, produces a latent prediction $\hat{z}_{t+k}$, given a summary $z_{\leq t}$ of observations up to time $t$. Our method similarly exploits the temporal structure of video data for contrastive sampling.

Several similarity metrics have been proposed in the literature. A common choice is the weighted dot-product $\exp(z_i^T W z_j)$ or $\sigma(z_i^T W z_j)$ \cite{chen2020simple, chen2020big}. The cosine distance and L2 distance have also proven effective~\cite{oord2018representation, yan2020learning, schneider2019wav2vec}. 

\subsection{Domain Randomization}\label{ssec:DR}
Domain randomization is a popular method of simulation-to-real transfer for neural networks \cite{sadeghi2016cad2rl,tobin2017domain}. By exposing the network to extensive variations of scene parameters that are irrelevant for decision making during training (e.g. lighting conditions and textures), the goal of DR is to learn robust features that will transfer well to the real world. Here, we cast DR under a probabilistic formulation that will serve our subsequent development.

We assume that an observation (either simulated or real) $o(x,e)$ is generated from two independent random variables $x$ and $e$, representing relevant and irrelevant domain properties, respectively. We further assume access to the observation generation process---i.e. given $x$ and $e$, we can generate $o(x,e)$. For example, $x$ can represent the pose of an object, $e$ its texture, and $o(x,e)$ is generated by a rendering engine. Under these assumptions, DR can be framed as a supervised learning problem over the following loss:
$
    \mathcal{L_{\text{DR}}}(l) = \mathbb{E}_e \mathbb{E}_x \mathbb{E}_{y|x,e} \left[l(o(x, e), y)\right],
$
where $y$ is a label for the observation $o(x,e)$, sampled from $p(y|x,e)$, and $l$ is some supervised learning loss function, such as regression or classification. If $y$ is independent of $e$, optimizing over this loss encourages the network to learn a prediction that is agnostic to $e$. 

\section{Contrastive Domain Randomization (CDR)}
In this section, we describe our proposed framework for learning domain-invariant representations: Contrastive Domain Randomization (CDR). We begin by presenting a formalism for incorporating DR into the contrastive learning framework. We then show that a straightforward combination of CPC with DR does not learn domain-invariant representations. Finally, we introduce our method for learning contrastive predictive models and show that, under mild assumptions, the learned representations of our model are in fact domain-invariant.

\begin{figure}[htbp]
    \centerline{\includegraphics[scale=0.22]{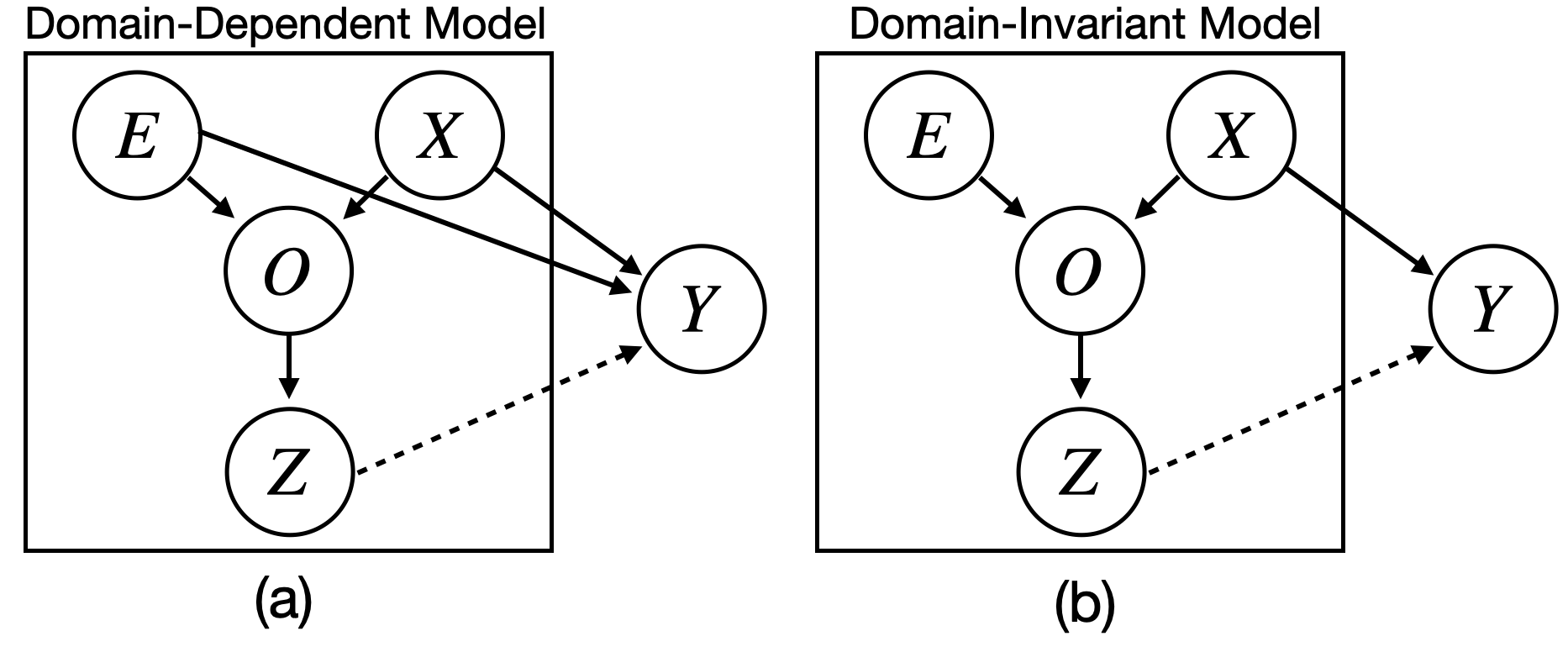}}
    \caption{A causal model that produces environment observations $o(x,e)$, encoding $z$ and labels $y$. (a) applying Naive DR to InfoNCE process where labels $y$ depend on both relevant $x$ and irrelevant environmental properties $e$. (b) Contrastive DR, in which $y \independent e$ by randomly replacing the environmental properties $e$ that generates $y$ to $e'$.}
    \vspace{-1em}
    \label{fig_causal_models}
\end{figure}

\subsection{Problem Formulation}
Recall that observations $o(x, e)$ are generated from two independent variables---relevant properties $x$ and irrelevant environmental properties $e$---and $z = f_\theta(o(x,e))$ is an encoded representation of $o(x,e)$. As in DR, we assume access to the observation generation process. 
Let $y(x,e)$ be a random variable we wish to predict from $o(x,e)$, such as the outcome of an action $a_t$ or an encoding of future observations.
This framework supports two types of causal models: (i) domain-dependent models in which $y$ is influenced by both $e$ and $x$ (Figure \ref{fig_causal_models}a); (ii) domain-independent models in which $x$ predicts $y$, regardless of interventions over $e$ (Figure \ref{fig_causal_models}b). In the next section, we show that straightforwardly applying domain randomization to the CPC objective results in domain-dependent representations.

\subsection{Naive CPC does not Learn Domain-Invariant Features}
A simple application of DR to CPC is to domain-randomize each sample in our data, and apply CPC on this data~\cite{yan2020learning}. The corresponding loss for a batch of size $N$ can be written as:
\begin{equation}
    \sum_{i=1}^{N} -\log  \frac{h(\hat{z}(x_i,e_i),y(x_i,e_i))}{\sum_{j = 1}^N h(\hat{z}(x_i,e_i),y(x_j,e_j))},
    \label{NAIVE_CPC_DR}
\end{equation}
where $x_i,e_i$ correspond to random observations $o(x_i,e_i)$ in the data, and we denote by $\hat{z}(x_i,e_i)$ and $y(x_i,e_i)$ the prediction and label that correspond to the observation.

Though Eq.~\eqref{CPC} suggests that CPC maximizes the mutual information between the representation $z$ and the observation $o$, when inspecting Eq.~\eqref{NAIVE_CPC_DR}, we note that there is nothing in the loss function preventing both $y$ and $z$ from becoming dependent on $e$. That is, Eq.~\eqref{NAIVE_CPC_DR} may learn to maximize mutual information between $e$ and $z$, effectively learning features that encode spurious properties of the random domain.
In fact, depending on the task, $e$ might serve well in distinguishing positive labels from contrastive labels!
Indeed, we empirically show that this naive CPC formulation learns representations that depend on the random domain.

An alternative approach to Eq.~\eqref{NAIVE_CPC_DR} is to contrast between different observations in the data that share the same domain $e$. This will prevent CPC from learning to discriminate between labels based on $e$. The corresponding 
loss for a batch is:
\begin{equation}\label{eq:CPC_DR_2}
    \sum_{i=1}^{N} -\log  \frac{h(\hat{z}(x_i,e),y(x_i,e))}{\sum_{j = 1}^N h(\hat{z}(x_i,e),y(x_j,e))},
\end{equation}
where $x_i,e$ correspond to random observations from the data that share the same domain. While this approach is appealing, it suffers from several technical problems. The first is that SGD based optimization tends to work well when samples in a batch are not correlated. The second is that this method requires generating at least $N$ different observations for each combination of domain properties $e$, which becomes prohibitively large when we want high variation in $e$.

\subsection{Domain-Invariant CPC}
We propose a novel contrastive loss function that learns domain-invariant representations while preserving high sample efficiency. Our method exploits domain randomization to remove the dependence between predictive variables $y$ and spurious domain features $e$.

First, we sample $\{x,e,y\}$ from the joint distribution $p(e)p(x)p(y|x,e)$. Using the observation generation process, we then generate both positive $D = \{o(x,e),y\}$ and intervened $D' = \{o',y'\}$ datasets; where $o' = o(x,e')$ and $y'=y(x, e')$, and $e'$ is randomly sampled. Our key observation is that since $y'$ is sampled from $p(y'|x,e')$, we ensure that $y'$ is independent of $e$, namely, $y' \independent e$. 
We next introduce the Contrastive Domain Randomization loss:
\begin{equation}
    \sum_{i=1}^{N} -\log  \frac{h(\hat{z}(x_i,e_i),y'(x_i,e'_i))}{\sum_{j = 1}^N h(\hat{z}(x_i,e_i),y'(x_j,e'_j))},
    \label{CDR_batch}
\end{equation}
\noindent which is optimized using both $D$ and $D'$. Note that, in contrast to naive CPC \eqref{NAIVE_CPC_DR}, CDR does not induce learning representations that depend on $e$. In comparison to Eq.~\ref{eq:CPC_DR_2}, however, we can sample from different domains in the same batch, leading to improved efficiency. 

\subsection{Domain-Invariance Guarantees}

We now show that CDR indeed learns domain invariant representations. We focus on neural-network encoders, and specifically assume that the encoder can be written as:
$
    f(x,e) = \sigma(W_x^T \phi_x(x) + W_e^T \phi_e(e) + b),
$
that is, the encoder has separable features for $x$ and $e$. In this case, the following proposition shows that the CDR loss will learn to ignore $e$. \newline

\newtheorem{proposition}{Proposition}
\begin{proposition}
Let $x$ and $e$ be independent random variables defined over spaces $X$ and $E$, respectively. Let $y \in Y$ be a random variable that satisfies $y \independent e$. Consider a family of functions $f$, parameterized by $W_x, W_e$, and $b$, that satisfy:
\begin{equation}
    f(x,e) = \sigma(W_x^T \phi_x(x) + W_e^T \phi_e(e) + b),
    \label{separable_function}
\end{equation}
\noindent where $\phi_x(x)$ and $\phi_e(e)$ are functions that depend solely on $x$ and $e$, respectively, and $\sigma$ is some activation function. Given the following loss:
\begin{equation}
    \mathcal{L}(f) = \mathbb{E}_e \mathbb{E}_x \mathbb{E}_{y|x} \left[l(f(x, e), y)\right],
    \label{Domain_Invariant_loss}
\end{equation}
\noindent an optimal set of parameters is achieved when $\phi_e(e)$ is nullified (i.e. $W_e = 0$).
\label{proposition_1}
\end{proposition}

\begin{proof} 
Our proof assumes that $X$ and $E$ are finite, but can be extended to the infinite setting. Define $g(e)$ as the expectation
\begin{equation*}
    g(e) = \mathbb{E}_x \mathbb{E}_{y|x} \left[l(f(x, e), y)\right].
    \label{g_e}
\end{equation*}
The objective in \eqref{Domain_Invariant_loss} can be restated as: $\mathcal{L}(f) = \mathbb{E}_e g(e)$. We can then define $e^* = \argmin_{e}{g(e)}$ as the optimal environmental variable that minimizes $\mathcal{L}(f)$ and write:

\begin{equation*}
    \mathcal{L}(f) = \mathbb{E}_e g(e) = \sum_e{p(e) g(e)}
\geq \sum_e{p(e) g(e^*)} = g(e^*).
\vspace{-0.5em}
\end{equation*}

For any parameters $W_x$ and $b$, consider the case in which $W_e = 0$. If we define a new parameter $b'$  as $b' = b + W_e^{*T} \phi_e(e^*)$, where $e^*, W_e^* = \argmin_{e, W_e}{g(e)}$, then we can re-write:
\begin{equation*}
    g(e) = \mathbb{E}_{x} \mathbb{E}_{y|x} \left[l(\sigma(W_x^T \phi_x(x) + b'), y)\right] = g(e^*).
\end{equation*}
\noindent It follows that for the set of parameters $\{W_x^{opt}, b^{opt}\}$ that minimize $\mathcal{L}(f)$, the parameters $\{W_x=W_x^{opt}, W_e=0, b=b'\}$ produce the same loss. 
\end{proof}

\vspace{0.5em}
If $f$ learns to produce $\phi_x(o(x,e)){=}\phi_x(x)$ and $\phi_e(o(x,e)) = \phi_e(e)$ in any layer of the neural network, then according to Proposition \ref{proposition_1} the optimal representation will be domain-invariant. While the assumption on separable features is idealistic, modern neural-network architectures are often expressive enough to be capable of separating $x$ and $e$. In such cases, the result above guarantees invariant features. In our experiments, we further demonstrate this result empirically.

\section{CDR Feature Learning from Videos} \label{CDR Videos}
In this section, we describe how to use CDR for learning domain-invariant image representations that capture intuitive physical concepts and generalize to the real world. We assume to have access to a physics simulator that follows the real world physical rules and renders image observations. We consider two experimental paradigms -- controlled and uncontrolled, as we next describe.

\subsection{Uncontrolled environment}\label{ss:uncontrolled}
In the uncontrolled setting we only observe sequences of images $\{o_t\}$ depicting some physical interaction. 

We follow CPC~\cite{oord2018representation} and define $g_\phi$ to be an auto-regressive model (GRU) that summarizes the past observations in the latent space $z_{\leq t}$ and predicts the next $K$ future latent states $\hat{z}_i = \hat{z}_{t+k} = W_k^T GRU(z \leq t)$. we set our positive labels $y'$ to be the true representation at time $t+k$, $z'_i = z'_{t+k} = f_\theta(o(x_{t+k},e'))$, and contrastive labels are drawn from different time-steps or from different videos. we average $K$ predictions and optimize: 

\begin{equation}
\sum_{i=1}^{N} \left[\frac{1}{K} \sum_{k=1}^{K} -\log  \frac{h(\hat{z}_{t+k},z'_{t+k})}{\sum_{j=1}^{N} h(\hat{z}_{t+k},z_j))}\right]. 
\label{CDR_uncontrolled}
\end{equation}

\subsection{Controlled environment}
In the controlled settings, an agent interacts with the environment and we can observe its actions. That is, we observe $o_t$ and an action $a_t$ at each time-step. Here, we define $g_\phi$ to be a 1-step forward model $\hat{z}_i = \hat{z}_{t+1} = g_\phi(z_t, a_t)$. We set our positive labels $y'$ to be the true representation of the next frame $z'_i = z'_{t+1} = f_\theta(o(x_{t+1},e'))$.

\begin{figure}[htbp]
\centerline{\includegraphics[scale=0.12]{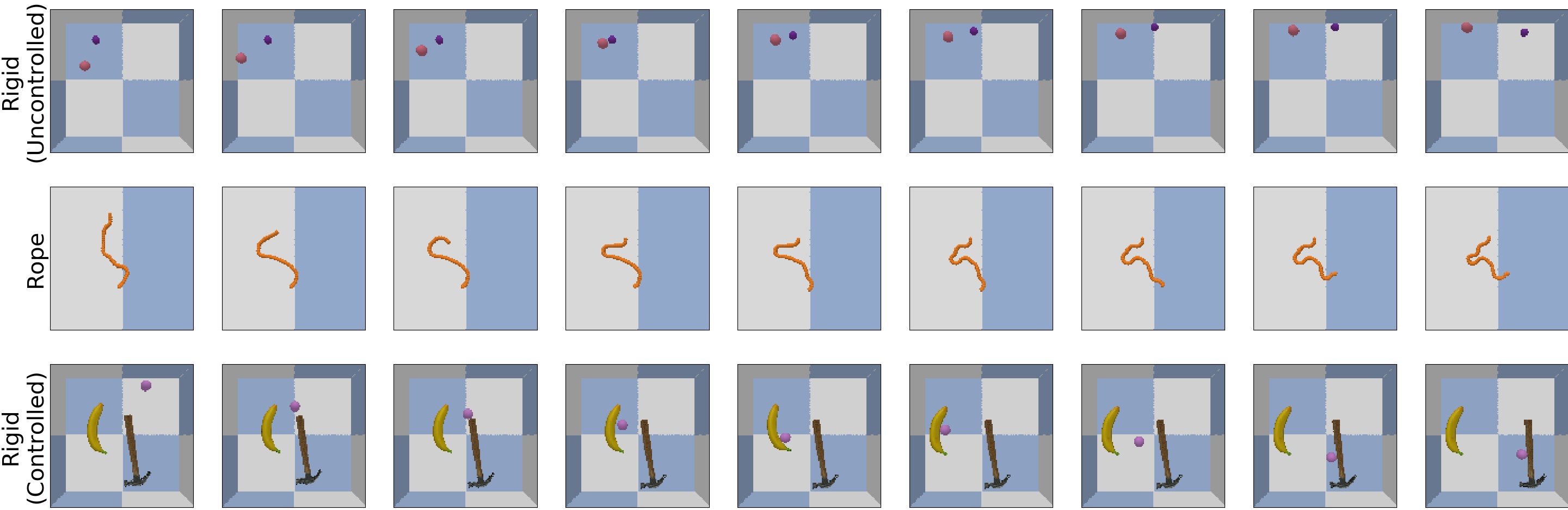}}
\caption{A short trajectory sampled from each simulated environment. Top: 2 objects colliding in a low friction environment. Middle: a Robotic arm manipulating a deformable rope. Bottom: a small cube agent pushing a hammer and a banana.}
\label{Fig:Trajectories}
\vspace{-1em}
\end{figure}

\section{Experimental Results}

In this section, we empirically evaluate  our method on both simulated and real world data. Our evaluations are built to answer the following questions: 

\noindent (1) Does CDR learn to produce both informative and domain-invariant representations? \newline
\noindent (2) Is CDR robust to the simulation-reality gap?

Since our representation learning is unsupervised, it is difficult to define a quantitative performance measure to evaluate the quality of the learned representations. To address this challenge, we devised both simulated experiments that measure the representation quality, and real-robot experiments that evaluate the representations when input to a downstream visual planning task. In simulation, we designed experiments in which certain physical properties, such as the position and orientation of objects, are necessary for making accurate predictions about the future evolution of the system. In such domains, we know that a good representation learning method must represent these features reliably, and we correspondingly design an evaluation protocol termed \textit{information retrieval} to reflect this. For the real-robot evaluation, we evaluate representations that were trained solely in simulation, by using them as input to an MPC-style planner based on image goals~\cite{yan2020learning}. We term this sim-to-real evaluation as \textit{planning}. 
We next describe our domains, training details, and results.

\begin{figure}[htbp]
\centerline{\includegraphics[width=\columnwidth]{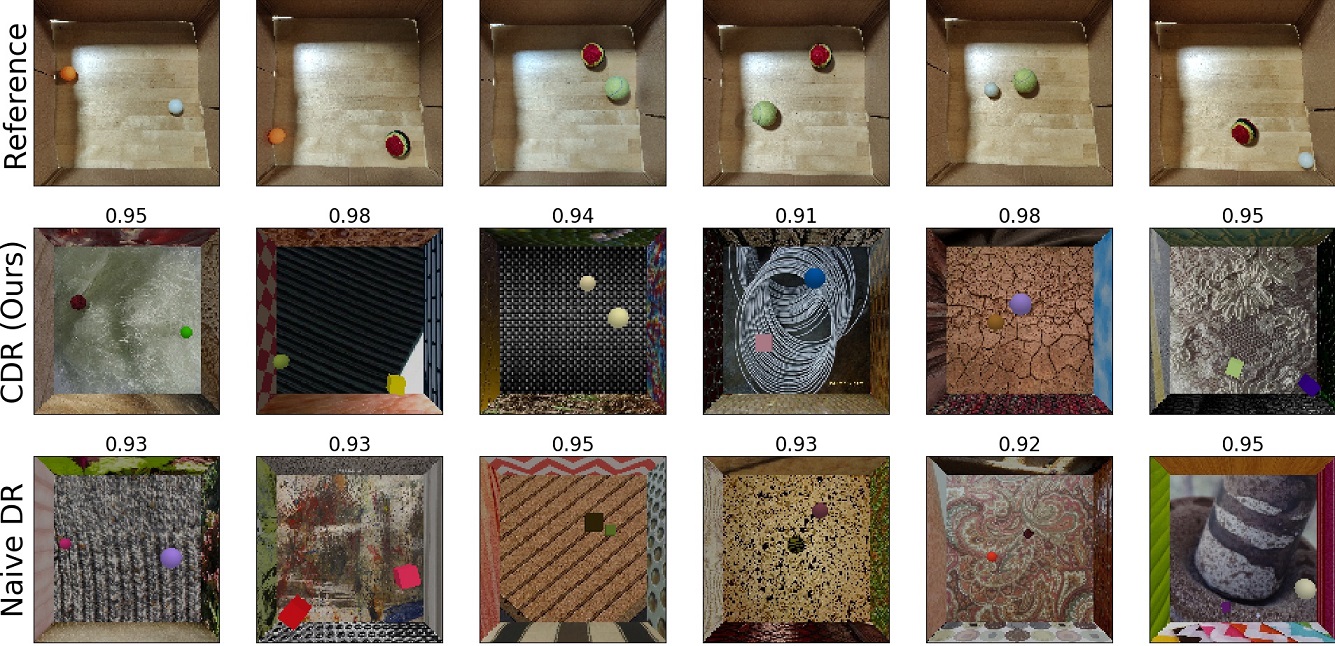}}
\caption{Qualitative comparisons of CDR and naive DR for Uncontrolled Rigid objects environment. We retrieve the nearest neighbour (in latent space) of a real-world OOD reference image from a held-out simulated dataset of 30K images.
we report the Cosine distance between the features of the reference image and the retrieved one. note that CDR latent representations retrieve images that are much more similar to the reference.}
\label{2 rigid objects}
\vspace{-1em}
\end{figure}

\begin{figure*}[htbp]
\centerline{\includegraphics[width=\textwidth]{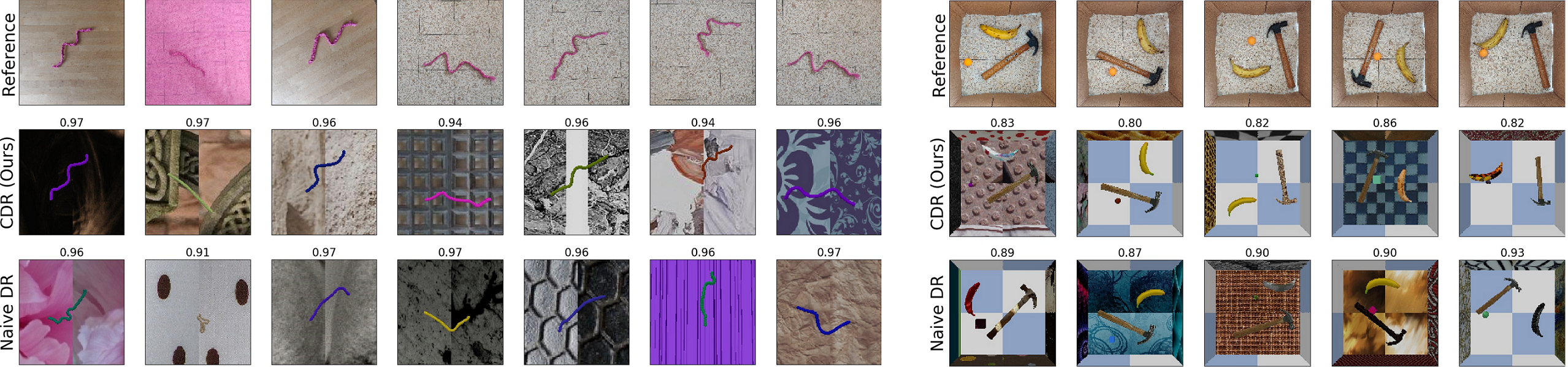}}
\caption{Qualitative comparisons of CDR and naive DR for Controlled environments. Left - Deformable rope, Right - Rigid objects. We retrieve the nearest neighbour (in latent space) of a real-world OOD reference image from a held-out simulated dataset of 30K images. we report the cosine distance between the features of the reference image and the retrieved one. note that CDR latent representations yield images that are much more similar to the reference.}
\label{Action based}
\end{figure*}

\begin{table*}[h]
\caption{Accuracy in Predicting Physical Properties}
\begin{center}
\begin{scriptsize}
\begin{tabular}{|c|c|c|c|c|c|c|c|c|}
\hline
 
 & \textbf{Euclidean} & \textbf{Euclidean (OOD)} & \multicolumn{3}{|c|}{\textbf{IOU}} & \multicolumn{3}{|c|}{\textbf{IOU (OOD)}} \\
 \hline
\textbf{} & \textbf{Rigid} & \textbf{Rigid} &  \textbf{Rigid} &
\textbf{Rigid} & \textbf{Rope} &  \textbf{Rigid} &
\textbf{Rigid} & \textbf{Rope} \\
\textbf{} & \textbf{uncontrolled} & \textbf{uncontrolled} & \textbf{uncontrolled} &
\textbf{controlled} & \textbf{} & \textbf{uncontrolled} &
\textbf{controlled} & \textbf{} \\
\hline
\textbf{Baseline} & $0.45 \pm 0.39$ & $0.55 \pm 0.46$ & $0.18 \pm 0.14$ & $0.19 \pm 0.13$ & $0.15 \pm 0.12$ & $0.15 \pm 0.14$ & $0.15 \pm 0.13$ &  $0.14 \pm 0.11$ \\
\hline
\textbf{CDR (Ours)} & $\textbf{0.3} \pm \textbf{0.36}$ & $\textbf{0.42} \pm \textbf{0.48}$& $\textbf{0.34} \pm \textbf{0.16}$ & $\textbf{0.27} \pm \textbf{0.12}$ & $\textbf{0.24} \pm \textbf{0.15}$  & $\textbf{0.28} \pm \textbf{0.17}$ & $\textbf{0.26} \pm \textbf{0.13}$ & $\textbf{0.22} \pm \textbf{0.13}$ \\
\hline
\end{tabular}
\end{scriptsize}
\end{center}
\label{IOU}
\end{table*}

\begin{table*}[h]
\caption{Invariance to Random Texture} 

\begin{center}
\begin{tabular}{|c|c|c|c|c|c|c|}
\hline
 & \multicolumn{3}{|c|}{\textbf{Cosine similarity (higher is better)}} & \multicolumn{3}{|c|}{\textbf{MSE distance (lower is better)}} \\
\textbf{} & \textbf{Rigid (uncontrolled)} & \textbf{Rigid (controlled)} & \textbf{Rope}  & \textbf{Rigid (uncontrolled)} & \textbf{Rigid (controlled)} & \textbf{Rope}\\
\hline
\textbf{Baseline} & $0.69 \pm 0.22$ & $0.8 \pm 0.19$ & $0.86 \pm 0.12$ & $0.51 \pm 0.44$ & $0.8 \pm 0.67$ & $1.06 \pm 0.89$ \\
\hline
\textbf{CDR (Ours)} & $\textbf{0.86} \pm \textbf{0.21}$ & $\textbf{0.91} \pm \textbf{0.1}$ & $\textbf{0.97} \pm \textbf{0.06}$ & $\textbf{0.13} \pm \textbf{0.2}$ & $\textbf{0.12} \pm \textbf{0.12}$ & $\textbf{0.11} \pm \textbf{0.24}$ \\
\hline
\end{tabular}
\end{center}
\label{domain_invariant_cosine_mse}
\end{table*}

\subsection{Simulated Environments}
We focus on the following environments (see Fig.~\ref{Fig:Trajectories}).
\paragraph{Uncontrolled Rigid objects}
Two simple rigid objects, a cube or a ball, initialized with random size, position and mass, are placed in a blocked frame with low friction. At the start of each episode a force in a random direction and magnitude is applied to one of the objects, making it collide with the frame and the second object over time. 

\paragraph{Controlled Rigid objects}
Two complex shaped rigid objects, a Hammer  and a Banana, initialized with random size, position and orientation, placed in a blocked frame with high friction. An agent, represented by a small cube or sphere initialized with random size, interacts with the environment by randomly pushing the other objects in a random direction. 
We use a two-dimensional action space: the direction and magnitude in Cartesian coordinates of the force applied to the agent at each time-step.

\paragraph{Controlled Deformable Rope}
A rope with random length and thickness, represented in simulation as $n \in \{35... 45\}$ connected small spheres with random radius, is being pushed in a random direction by a robotic arm. The arm is not visible in the overhead images. We use a four-dimensional action space: the 2-d initial and final positions of the robot end effector. At the start of each episode, the rope state is randomly initialized by applying a random number of random actions. \newline

All our experiments use the bullet~\cite{coumans2013bullet} physics simulator with an overhead camera that renders 128×128×3 RGB images as observations. We apply random textures sampled from Describable Textures Dataset (DTD)~\cite{cimpoi2014describing}. 

\begin{figure}[htbp]
\centerline{\includegraphics[width=0.6\linewidth]{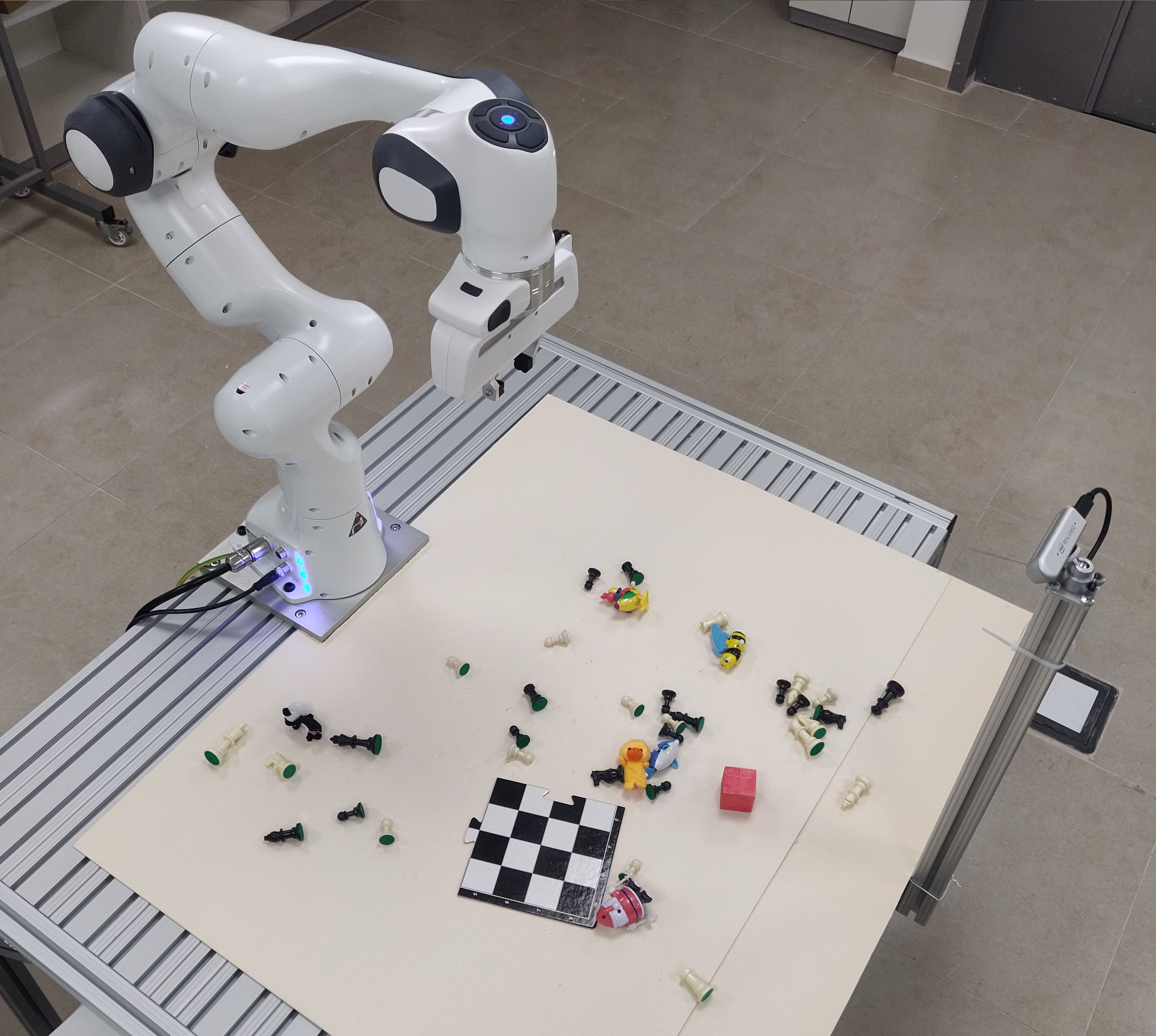}}
\caption{Real Environment, a Panda arm needs to push the red cube toward a goal position represented as an image.}
\label{panda_set_up}
\end{figure}

\subsection{Real Robot Environment}
To evaluate sim-to-real transfer, we use the Franka Emika Panda robotic arm, with an Intel RealSense D415 camera positioned to look diagonally down at the scene (see Fig.~\ref{panda_set_up}). We make sure that the arm is not visible in the images by moving it out of sight before taking an image.
We also use a simulation of the real Panda robotic arm environment to generate training data and for evaluation. Our manipulation task is moving a random sized cube. The actions are similar to the Controlled Deformable Rope environment -- the Panda robotic arm pushes the cube in a random direction that is parallel to the table, specified by a start and goal position for the tip of the gripper.

\begin{figure}[htbp]
\centerline{\includegraphics[width=\linewidth]{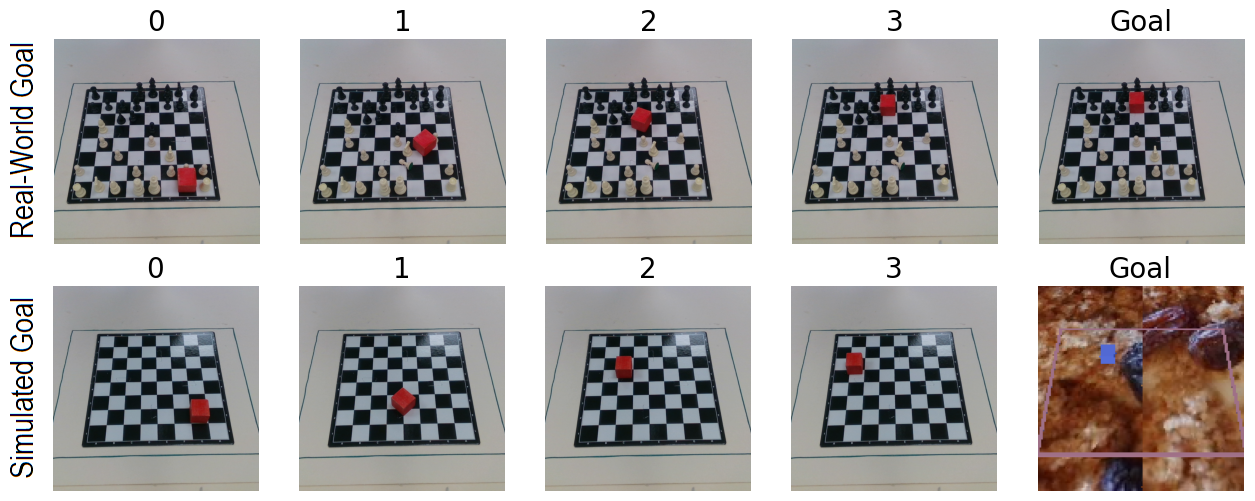}}
\caption{Samples of trajectories using CDR representations. Top - goal image is generated from the real domain. Bottom - goal image is generated in simulation.}
\label{Fig:Planning_golas}
\end{figure}

\begin{figure*}[htbp]
\centerline{\includegraphics[width=\textwidth]{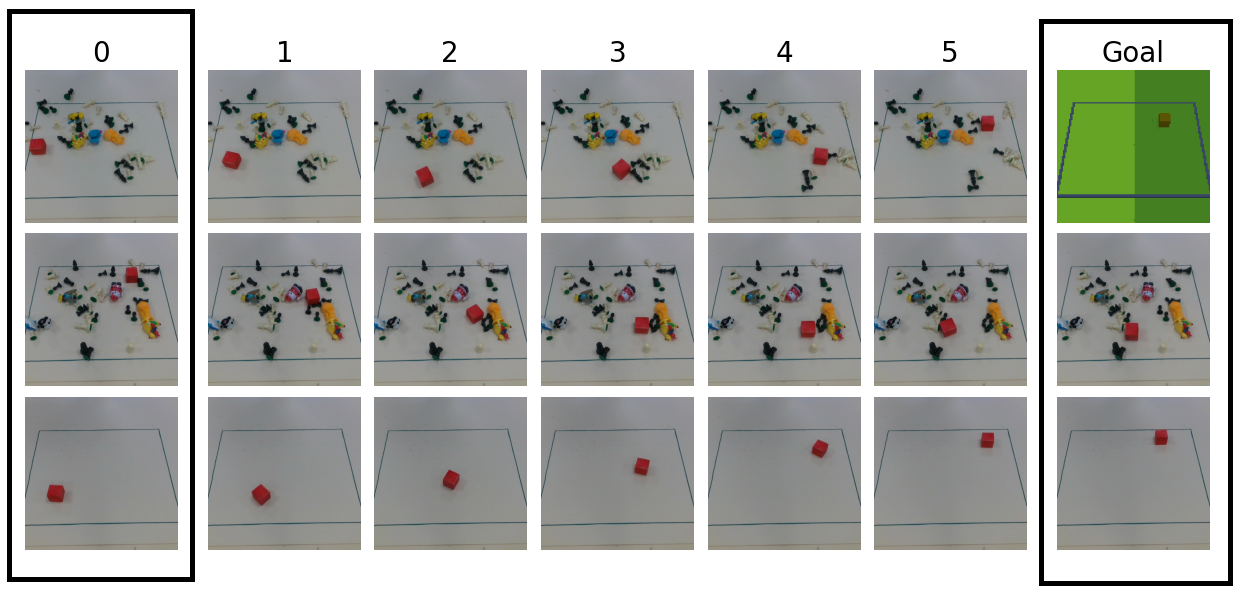}}
\caption{Planning results: each row is a trajectory sampled by planning with CDR representations. In each frame, a 1-step MPC planner selects an action, which is then executed by the Panda arm. Goal images can be generated in simulation, while the various real-world environments are dramatically different from the simulated training environments.}
\label{Fig:Planning_samples}
\end{figure*}

\subsection{Training and Baselines}\label{ssec:training}

\paragraph{Real Robot Experiment}
for fair comparison with the state of the art, we use Contrastive Forward Modeling (CFM)~\cite{yan2020learning} as baseline. Both baseline and CDR uses the same model, a ResNet18 \cite{he2016deep} that was pre-trained on the ImageNet \cite{deng2009imagenet} dataset as a backbone encoder, CFM forward model architecture, and L2 similarity function $h_{i,j} = -||z_i - z_j||^2$. We replace the last fully connected (FC) of the encoder with a FC layer that compresses the encoder outputs to 8-dimensional latent vector $z$.\footnote{Our code is publicly available at https://github.com/carmelrabinov/cdr}
  
We applied DR and generated a set of 4,000 videos with 15 frames each, where each set samples textures independently from DTD. Baseline is trained with the objective presented in Eq.~\ref{NAIVE_CPC_DR} as in CFM, and CDR with the objective presented in Eq.~\ref{CDR_batch}.
We also use 10\% of the training videos as a validation set for early stopping.

\paragraph{Simulated Experiments}
For the Controlled environments (Rigid and Deformable) we use dot-product similarity $h_{i,j} = \exp(z_i^T z_j)$ and replace the CFM forward model architecture with an action encoder that encodes the actions to a 16-dimensional vector with a 2-layered MLP with tanh activations and a hidden layer of size 64 for both Baseline and CDR. the encoding of the state $z_t$ and actions are then concatenated and used as input to a 3-layered MLP with ReLU activations and a skip connection. we found this architecture beneficial for both Baseline and CDR.

For the Uncontrolled Rigid objects environment, we use bilinear similarity and replace the CFM forward model architecture with a GRU that summarize past observations and predict $K$ steps into the future as in CPC, i.e $\hat{z}_{t+k} = W_k^T GRU(z \leq t)$. we use the objective in Eq.~\ref{CDR_uncontrolled} with $K=6$.

We generated a set of 10,000 videos with 30 frames each for both rigid environments, and a set of 8,000 videos with 15 frames each for the deformable rope environment. DTD is divided into 47 categories, we use 3 categories as out-of-distribution (OOD) test textures and sample uniformly from the rest during training.

\subsection{Information Retrieval Results}

Our goal is to learn representations that capture relevant properties of the scene, and ignore spurious features. To evaluate such behavior on out-of-distribution data, we test the ability to retrieve similar observations that share similar physical properties (e.g. position, orientation, size) but from different domains. Given a test image $o$, we first encode it into $z=f_\theta(o)$, and retrieve the latent nearest neighbour from an external dataset $D_{ext}$, unseen during training $z_{NN} = \argmin_{z'\in D_{ext}} h(z, z')$. For $h$, we use the Cosine distance, as we found it to perform well for both CDR and the baseline.
We evaluate test images from both in distribution (following the terminology of Sec.~\ref{ssec:DR}, with $e$ seen in training but different $x$) and OOD (both $e$ and $x$ were not seen in training). To evaluate performance, we measure the Intersection Over Union (IOU) of the object pixels between the nearest neighbor and the test image (higher is better). For the rigid object environments, we also report the sum of Euclidean distances between the objects in the test image and its retrieved nearest neighbour (lower is better). Table~\ref{IOU} shows  quantitative  results of CDR against Baseline. Note that CDR significantly outperforms the baselines by a large margin. 

To test if our representations are in fact domain-invariant, we randomly sample observations from different domains but with the same object configurations $o_1 = o(x,e_1), o_2 = o(x,e_2)$, and compare the distance between them in the latent space, i.e $h(o_1, o_2)$. Table~\ref{domain_invariant_cosine_mse} shows that CDR is significantly more domain-invariant compare to the baseline.

To evaluate sim-to-real transfer, we use real world images that are very different from the simulated data in both domain appearance and physical properties (see Fig.~\ref{2 rigid objects} and Fig.~\ref{Action based}). We then qualitatively evaluate our ability to retrieve similar simulated data. In Fig.~\ref{2 rigid objects} we compare CDR and Baseline to retrieve observations from a challenging reference image with a very different lightning and frame. Though retrieved positions are not always perfect, CDR tends to capture correctly the relative position and sizes of the objects -- an essential property for predicting the outcome of a collision.
For the controlled rigid objects data, we use a hammer and a banana much bigger than the maximal size in simulation. Nevertheless, CDR tends to capture correctly the relationship between the objects and the cube agent.
For the rope data, we use a colored rope, and unlike previous work \cite{yan2020learning}, we use a long rope that can form complex shapes. Furthermore, to demonstrate generalization outside lab conditions, we also use the same rope material as background for some reference images. Note that CDR yields retrievals that are much more accurate than the baseline.

\begin{table*}[h]
\caption{Euclidean distance from final state to goal state in cm. each simulation is given 10 trials to reach a goal. Simulation results are based on 80 experiments, Real robot results with goal image from real world or simulation are based on 16 and 21 experiments, respectively.}

\begin{center}
\begin{tabular}{|c|c|c|c|c|}
\hline
 & \multicolumn{2}{|c|}{\textbf{Simulation}} & \multicolumn{2}{|c|}{\textbf{Real Robot}} \\
\textbf{} & \textbf{In domain goal} & \textbf{Different domain goal} & \textbf{Real world goal} & \textbf{Simulated goal} \\
\hline
\textbf{Baseline} & $ 15.18 \pm 9.01 $ & $ 15.47 \pm 8.20 $ & $15.26 \pm 9.13$ & $10.02 \pm 5.88$ \\
\hline
\textbf{CDR (Ours)} & $\textbf{9.31} \pm \textbf{7.93}$ & $\textbf{7.93} \pm \textbf{7.46}$ & $\textbf{7.59} \pm \textbf{8.04}$ & $\textbf{3.48} \pm \textbf{3.34}$ \\
\hline
\end{tabular}
\end{center}
\label{planning_table}
\end{table*}

\subsection{Planning Results}
To evaluate if representations learned with CDR are useful for planning, we test our model in reaching a goal image from a random state using a simple 1-step Model Predictive Control (MPC) planner. The planner is provided with a goal image $o_{goal}$ which can be either a real-world image or a simulated one (see Fig.~\ref{Fig:Planning_golas}), which is then encoded into $z_{goal} = f_{\theta}(o_{goal})$. In each time step, the planner samples a set $A$ of 1000 possible actions that are guaranteed to push the cube in some random direction and magnitude. We feed these into our forward model, along with the representation of the current image $z_t=f_{\theta}(o_t)$ to get a set of 1000 latent representations of next states, corresponding to the different actions, $\hat{z}_{t+1} = g_{\phi}(z_t, a)$. We greedily choose an action to execute, based on distance in latent space $a_t = \argmin_{a\in A} h(z_{goal}, g_{\phi}(z_t, a))$.

To evaluate sim2real transfer, we use both simulated goal images and real-world goal images, and test the robot on a variety of challenging environments. For example, we add multiple objects to the scene, while no additional objects were seen during training (see Fig.~\ref{Fig:Planning_samples}).  
We measure the Euclidean distance from the final state to the goal state. 
Table~\ref{planning_table} shows that CDR significantly outperforms the CFM Baseline.

\section{Conclusion}
Learning from simulated data is a promising approach to scale up robot learning to complex tasks. In this work, we proposed a principled robustification of unsupervised representation learning using domain randomization, and demonstrated that it can learn relevant representations that are robust to irrelevant features of the domain appearance. Key to our approach is exploiting the fact that we can intervene on a simulated image, and change some of its irrelevant features.

Our framework is general, and can be used with any physical simulator, any domain randomization technique, and any InfoNCE based representation learning method. As such, it is a promising direction for learning general representations for robotic tasks.

\section*{Acknowledgments}
The authors wish to thank Shadi Endrawis for helping in setting up the Panda arm interface. Aviv Tamar is partly funded by the Israel Science Foundation (ISF-759/19)

\bibliographystyle{./bibliography/IEEEtran}
\bibliography{./bibliography/cdr}

\end{document}